\documentclass{article}

\usepackage{amsmath}
\usepackage{amssymb}
\usepackage{graphicx}
\usepackage{algorithm2e}

\usepackage{hyperref}

\usepackage{rotating}
\usepackage{bm}
\usepackage{multirow}

\usepackage{amsthm}
\newtheorem{lemma}{Lemma}

\usepackage[subrefformat=parens]{subcaption}
\title{
Least Angle Regression in Tangent Space and 
LASSO for Generalized Linear Models\thanks{
This work was partly supported by 
JSPS KAKENHI Grant Number JP18K18008 
and JST CREST Grant Number JPMJCR1763.
}
}

\author{
  Yoshihiro Hirose\thanks{hirose@ist.hokudai.ac.jp}\\
  Hokkaido University, Japan
}

\date{}

\begin{document}

\maketitle

\begin{abstract}
This study proposes sparse estimation methods for the generalized linear models, 
which run one of least angle regression (LARS) and 
least absolute shrinkage and selection operator (LASSO) 
in the tangent space of the manifold of the statistical model. 
This study approximates the statistical model 
and subsequently uses exact calculations. 
LARS was proposed as an efficient algorithm for parameter estimation and variable selection for the normal linear model. 
The LARS algorithm is described in terms of Euclidean geometry regarding the correlation as the metric of the parameter space. 
Since the LARS algorithm only works in Euclidean space, 
we transform a manifold of the statistical model into the tangent space at the origin. 
In the generalized linear regression, 
this transformation allows us to run the original LARS algorithm 
for the generalized linear models. 
The proposed methods are efficient and perform well. 
Real-data analysis indicates that the proposed methods output similar results 
to that of the $l_1$-regularized maximum likelihood estimation for the aforementioned models. 
Numerical experiments reveal that  
our methods work well 
and they may be better than the $l_1$-regularization 
in generalization, parameter estimation, and model selection. 
\end{abstract}

\section{Introduction}
\label{intro}

We propose sparse estimation methods for generalized linear models (GLM). 
One of the proposed methods is based on least angle regression (LARS) 
\cite{EHJT2004} 
and is described in terms of information geometry. 
The main features of our approach are 
i) we use an approximation of a statistical model and do not use the statistical model itself, 
and ii) the proposed methods are 
calculated exactly, 
which allows us to compute the estimators efficiently. 
A few extensions of LARS, which are based on 
information, Riemannian, and differential geometry, 
have been proposed in the literature such as 
\cite{HK2010} and \cite{AMW2013}, for example.  
The existing methods take advantage of a dual structure of a model manifold, 
which requires computational costs. 
Our method utilizes a part of the dual structure 
and uses the original LARS algorithm in the tangent space. 
The proposed method enables us to compute the estimator easily. 
Furthermore, 
we demonstrate that least absolute shrinkage and selection operator 
(LASSO) 
\cite{T1996} 
for the normal linear model is also available in the tangent space.

Sparse modeling has been  extensively investigated 
in this two decades. 
As a representative method, 
LASSO has motivated many researchers in statistics, machine learning, and other fields. 
LASSO was proposed 
as an estimation and variable-selection method for the normal linear model. 
It minimizes the $l_1$-regularized least square with a tuning parameter. 
Various generalizations have been proposed for other problems. 
For example, 
\cite{PH2007} and \cite{YL2007} 
treat the generalized linear regression
and Gaussian graphical models, respectively. 
See also 
\cite{HTF2009}. 

LARS was proposed for the same problem as LASSO. 
The LARS algorithm is very efficient, 
and can also compute the LASSO estimator if a minor change is added. 
The LARS algorithm uses only correlation coefficients between the response and explanatory variables. 
The algorithm is therefore described in terms of Euclidean geometry.

Information geometry is a Riemannian-geometric framework for statistics and other fields 
\cite{A1985, A2016, AN2000, AJLS2017, KV1997, MR1993}.  
In this framework, 
we treat a statistical model as a Riemannian manifold and 
take advantage of its geometrical properties for estimation, test, and other tasks. 
Each probability distribution is treated as a point in the manifold. 
For example, 
estimation problem for the generalized linear regression 
can be described in terms of the geometry. 
The GLM is treated as a manifold and 
an estimator assigns a point therein to an observed data. 
The maximum likelihood estimator (MLE) uses a kind of projection.

Some extensions of LARS have been proposed based on the information geometry 
of the exponential family of distributions. 
\cite{HK2010} and \cite{AMW2013} 
proposed different extensions of LARS, 
that take advantage of the dual structure of the model manifold. 
Their works are theoretically natural and can be extended to other models than the GLM 
\cite{HK2013, HK2015}. 
However, the existing methods require many iterations of approximation computation, 
which is inevitable for treating more complicated objects than Euclidean space. 
For example, 
\cite{AMW2013} 
treated many tangent spaces that each to an estimate 
while our methods use only one tangent space. 
\cite{AMW2013} 
stated 
``DGLARS method may be computationally more expensive than other customized techniques'' 
for the $l_1$-regularization method. 
We aim to provide as an efficient method as the $l_1$-regularization for the GLM. 
It should be noted that 
our approach is different from that of existing methods. 
We approximate the model manifold with only one tangent space 
and use the exact computation of LARS in the space. 
This approximation is natural from the viewpoint of information geometry. 
The usefulness of our idea is validated by numerical experiments. 
One advantage of our methods is that they do not require additional implementation 
because we can use existing packages. 

The rest of this paper is organized as follows. 
In section \ref{sec:pre}, 
we introduce our problem and the related works. 
In Section \ref{sec:larsts}, 
we propose a sparse estimation method based on LARS. 
Furthermore, LASSO-type estimators are also proposed. 
We compare our methods with the $l_1$-regularization for the GLM 
by performing numerical experiments 
in section \ref{sec:numerical}. 
Finally, our conclusions are presented in section \ref{sec:conclusion}. 
Lemmas and remarks are presented in Appendix \ref{sec:app}.

\section{Problem and Related Method}
\label{sec:pre}

In subsection \ref{subsec:problem}, 
we formulate the problem and introduce our notation. 
In subsections \ref{subsec:lars} and \ref{subsec:lasso}, 
we briefly describe the LARS algorithm and the LASSO estimators, respectively.

\subsection{Problem and Notation}
\label{subsec:problem}

We consider the generalized linear regression, 
which is an estimation problem of an exponential family of probability distributions \cite{A1985, B1986, MN1989}. 
In this regression, 
the expectation $\mu$ of a response $y$ is represented 
by a linear combination of explanatory variables $x_1, x_2, \dots, x_d$ as 
\begin{equation*}
h(\mu^a) = \sum_{i=1}^d x_i^a \theta^i, \,\,(a=1,2,\dots,n), 
\end{equation*}
where $a$ is the index indicating $a$-th sample, 
$h: \mathbb{R} \rightarrow \mathbb{R}$ is a link function, 
$n$ is the sample size, $d$ is the number of the explanatory variables, 
and ${\bm \theta}=(\theta^i)$ is the parameter to be estimated. 
Let $X=(x^a_i)$ be the design matrix, which is an $(n \times d)$-matrix. 
Let ${\bm y} = (y^a)$ and ${\bm \mu} = (\mu^a)$ be the response vector and its expectation, respectively, which are column vectors of length $n$. 

In general, the link function is a function of $\mu^a$ and is not determined uniquely. 
However, in the current study, we focus on the canonical link function, 
which results in useful properties of the exponential family. 
Our method is based on this assumption. 

In terms of probability distributions, 
the aforementioned problem corresponds to estimation for an exponential family of distributions, 
\begin{align}
\label{eq:glm}
{\cal M} = \{ p(\cdot|\,{\bm \theta}) |\, {\bm \theta} \in \mathbb{R}^d \}, \\
p(y|\, {\bm \theta}) 
= p(y|\, X, {\bm \theta}) 
= \exp \left\{ {\bm y}^\top X {\bm \theta} - \psi({\bm \theta})  \right\}, \notag
\end{align}
where $\psi: \mathbb{R}^d \rightarrow \mathbb{R}$ is a potential function. 
Our notation takes after that of \cite{A1985} and \cite{B1986}. 
For example, 
this formulation includes 
logistic and Poisson regressions. 

As a special case, 
the normal linear regression uses the link function $h(y) = y$ and a quadratic function as the potential function. 
Another example is the logistic regression, 
where the link function is $h(y) = y/(1-y)$ and the potential function is 
$\psi({\bm \theta}) = \sum_{a=1}^n \log\{ 1+\exp(\sum_{i=1}^d x^a_i \theta^i) \}$.

Herein, we assume that the design matrix $X$ is normalized, 
that is, each column vector has the mean zero and the $l_2$-norm one: 
$\sum_{a=1}^n x^a_i = 0$ and $\sum_{a=1}^n (x^a_i)^2 = 1$ for $i=1,2,\dots,d$. 
Furthermore, we assume that column vectors of $X$ are linearly independent.

\subsection{LARS}
\label{subsec:lars}

We briefly describe the LARS algorithm. 
In subsection \ref{subsec:larsts}, 
we use the LARS algorithm for proposing an estimation method. 
The detail and further discussions on LARS can be found in 
\cite{EHJT2004} and \cite{HTF2009}, for example. 

LARS was proposed 
as an algorithm for parameter estimation and variable selection in the normal linear regression. 
In the LARS algorithm, 
the estimator moves from the origin ${\bm \theta}=0$ 
to the MLE $\hat{\bm \theta}_{\mathrm{MLE}}$ of the full model. 
The full model refers to the linear model including all the explanatory variables. 
The MLE $\hat{\bm \theta}_{\mathrm{MLE}}$ is determined by the design matrix $X$ and the response $\bm y$. 
The detailed algorithm of LARS is shown in Algorithm \ref{algo:lars}, 
where $\hat{\bm \theta}_{(k)}$ is $k$-th estimate the algorithm outputs. 
After $d$ iterations, LARS outputs a sequence of the estimates 
$\hat{\bm \theta}_{(0)}, \hat{\bm \theta}_{(1)}, \dots, \hat{\bm \theta}_{(d)}$. 
\begin{algorithm}
\caption{The least angle regression (LARS) algorithm}
\SetAlgoLined
\KwData{the design matrix $X$ and the response vector $\bm y$}
\KwResult{the sequence of the LARS estimates $(\hat{\bm \theta}_{(k)})_{k=0,1,\dots,d}$ }
Initialization: $k:=1, \hat{\bm \theta}_{(0)}:=0, \hat{\bm \theta}_{(d)}:=\hat{\bm \theta}_{\mathrm{MLE}}, 
{\bm r}_{(0)} := \hat{\bm \theta}_{(d)} - \hat{\bm \theta}_{(0)} = \hat{\bm \theta}_{\mathrm{MLE}}$ \\
\While{$k < d$}{ 
Calculate the correlations $\hat{\bm c}_{(k)}$ and the active set $I_{(k)}$ of the indices: 
\begin{align*}
\hat{\bm c}_{(k)} := X^\top X {\bm r}_{(k-1)}, \,
\hat{C}_{(k)} := \max_{j}\{ |\hat{c}_{(k),i}| \}, \\
I_{(k)} := \{ i |\, |\hat{c}_{(k),i}| = \hat{C}_{(k)} \}. 
\end{align*}
Using $s_i = \mathrm{sign}\{ \hat{c}_{(k),i} \}\,\, (i \in I_{(k)})$, 
define a bisector of an angle ${\bm w}_{(k)}$ and others: 
\begin{align*}
X_{(k)} := (\dots s_i x_i \dots)_{i \in I_{(k)}}, \,\,
G_{(k)} := X_{(k)}^\top X_{(k)}, \\
A_{(k)} := ({\bm 1}_k^\top G_{(k)}^{-1} {\bm 1}_k)^{-1/2}, \, 
{\bm w}_{(k)} := A_{(k)} G_{(k)}^{-1} {\bm 1}_k, \\
{\bm a}_{(k)} := X^\top X_{(k)} {\bf w}_{(k)}. 
\end{align*}
Define the next estimate $\hat{\bm \theta}_{(k)}$ as 
\begin{align*}
\begin{cases}
(\hat{\theta}_{(k)}^i)_{i \in I_{(k)}} := (\hat{\theta}_{(k-1)}^i)_{i \in I_{(k)}} + \gamma \mathrm{diag}(s_i)_{i \in I_{(k)}} {\bm w}_{(k)}, \\
(\hat{\theta}_{(k)}^i)_{i \in I_{(k)}^{\mathrm{c}}} := 0
\end{cases}
\end{align*}
with 
\begin{align*}
\hat{\gamma} := \min_{j\in I_{(k)}^c}{}^{+} \left\{ \frac{\hat{C}_{(k)}-\hat{c}_{(k),j}}{A_{(k)}-a_{(k),j}}, \frac{\hat{C}_{(k)}+\hat{c}_{(k),j}}{A_{(k)}+a_{(k),j}} \right\} > 0, 
\end{align*}
where $\min^+ \{ a_1,\dots,a_N\} := \min\{a_i |\, a_i >0 \,(i=1,\dots,N)\}$. 
\\
Set ${\bm r}_{(k)} := \hat{\bm \theta}_{(d)} - \hat{\bm \theta}_{(k)}$ and $k:=k+1$. 
}
\label{algo:lars}
\end{algorithm}

The LARS algorithm is presented in Figure \ref{fig:lars}. 
Figures \ref{fig:lars_f} and \ref{fig:lars_b} indicate 
the estimator's move and the residual's move, respectively, 
in the parameter space $\mathbb{R}^d$ when $d=2$. 
The LARS estimator
i) selects an element of the parameter that forms a least angle 
between the residual ${\bm r}_{(k)} = \hat{\bm \theta}_{\mathrm{MLE}}-\hat{\bm \theta}_{(k)}$ 
and $\theta^i$-axis, 
and ii) uses it as a trajectory in the form of the bisector of an angle. 
The LARS algorithm is described in terms of Euclidean geometry and 
can be computed efficiently. 
Furthermore, $X^\top X$ plays an important role in the LARS algorithm, 
which is one of our motivations for considering the tangent space of a statistical model.

\begin{figure}[t]
 \begin{minipage}{0.5\hsize} 
  \begin{center}
  \includegraphics[width=\textwidth]{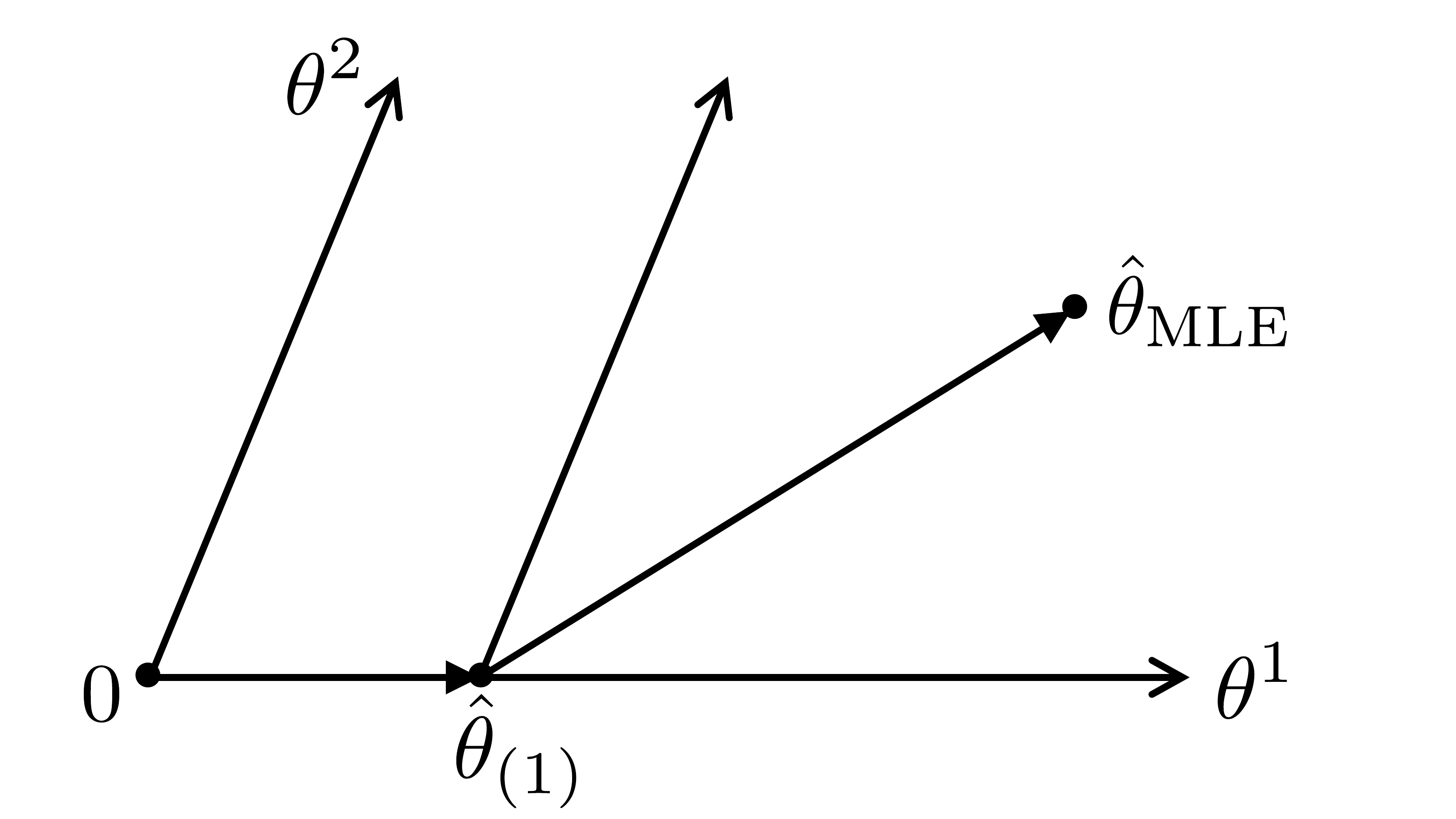}
  \end{center}
  \vspace{-5mm}
  \subcaption{The move of the estimator}
  \label{fig:lars_f}
 \end{minipage}
 \begin{minipage}{0.5\hsize} 
  \begin{center}
  \includegraphics[width=\textwidth]{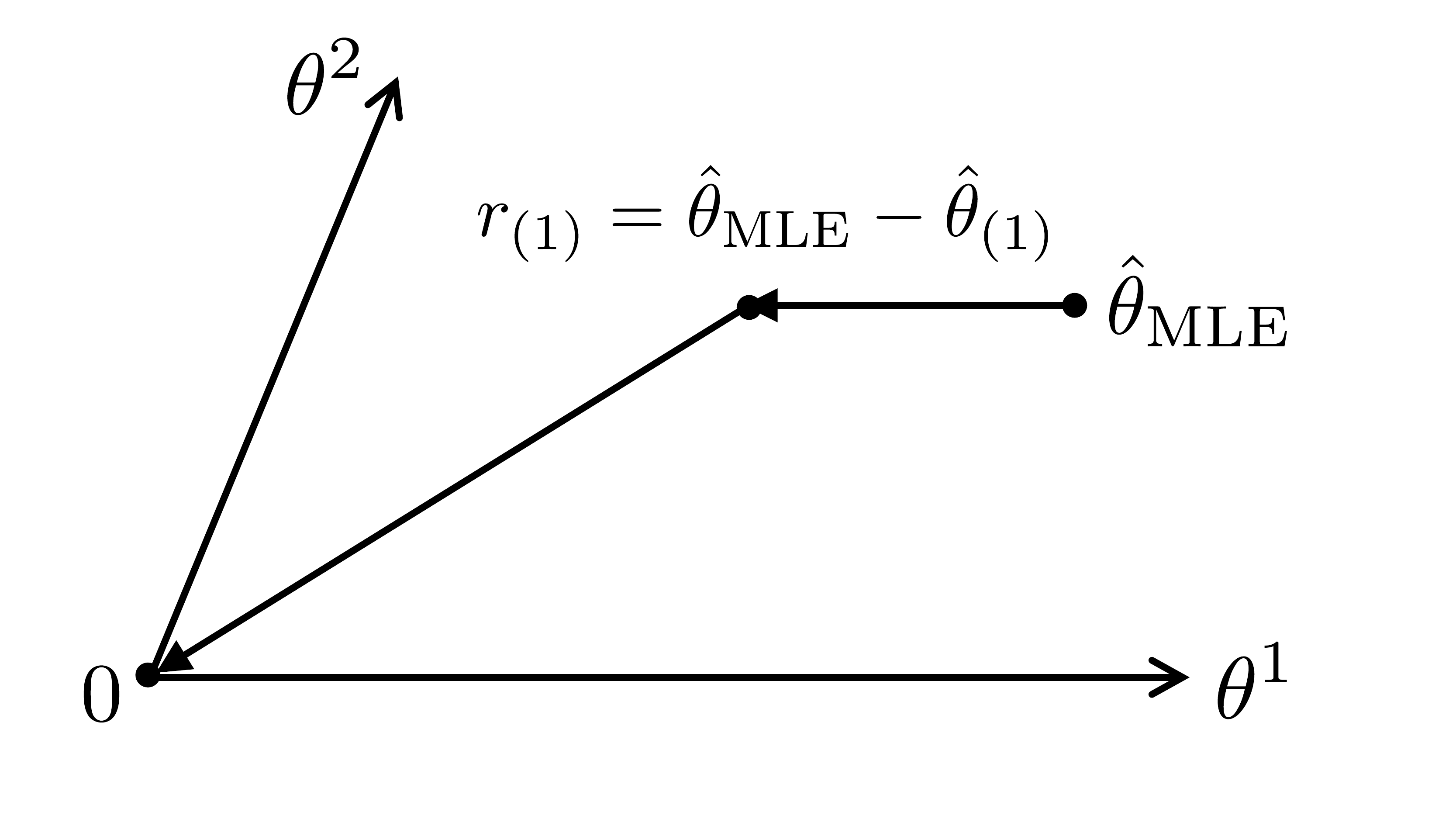}
  \end{center}
  \vspace{-5mm}
  \subcaption{The move of the residual}
  \label{fig:lars_b}
 \end{minipage}
 \caption{
The LARS algorithm when there are two explanatory variables. 
The parameter space is $\mathbb{R}^2$. 
$\hat{\bm \theta}_{\mathrm{MLE}}$ is the MLE of the full model. 
In this example, $\theta^1$ is selected at first iteration, $I_{1} = \{1\}$.  
The first estimate is $\hat{\bm \theta}_{(1)}$ and its second element is zero. 
The second estimate is $\hat{\bm \theta}_{(2)} = \hat{\bm \theta}_{\mathrm{MLE}} 
= \hat{\bm \theta}_{(1)} + \hat{\gamma}{\bm w}_{(2)}$. 
In Figure \ref{fig:lars_f}, 
the estimator moves along the bisector of an angle from $\hat{\bm \theta}_{(1)}$ 
to the second estimate $\hat{\bm \theta}_{(2)}$. 
Figure \ref{fig:lars_b} is another interpretation of the LARS algorithm. 
The residual ${\bm r}({\bm \theta}) = \hat{\bm \theta}_{\mathrm{MLE}} - {\bm \theta}$ 
moves from $\hat{\bm \theta}_{\mathrm{MLE}}$ to $0$ 
 }
\label{fig:lars}
 \end{figure}

\subsection{LASSO}
\label{subsec:lasso}

LASSO is an optimization problem for parameter estimation 
and variable selection in the normal linear regression. 
It solves the minimization problem 
\begin{equation*}
\label{eq:lasso1}
\min_{{\bm \theta} \in \mathbb{R}^d} 
\left\{ \| {\bm y}-X{\bm \theta} \|_2^2 + \lambda \| {\bm \theta} \|_1 \right\}, 
\end{equation*}
where $\lambda \geq 0$ is a tuning parameter. 
The path of the LASSO estimator when 
$\lambda$ 
varies can be made using the LARS algorithm with a minor modification.

LASSO can be applied to the GLM as the $l_1$-regularized MLE, 
which is the minimization problem 
\begin{align}
\label{eq:lassoglm}
\min_{{\bm \theta} \in \mathbb{R}^d} 
\left\{ -{\bm y}^\top X {\bm \theta} + \psi({\bm \theta})  + \lambda\| {\bm \theta} \|_1\right\}. 
\end{align}
For example, see 
\cite{PH2007}.

\section{The Proposed Methods}
\label{sec:larsts}

Our main idea, which is very simple, is to run the LARS algorithm in the tangent space of the model manifold. 
Although the idea appears to be extremely simple, 
it works well as is illustrated in sections \ref{sec:larsts} and \ref{sec:numerical}. 

In subsection \ref{subsec:ig}, 
we introduce information geometry used herein. 
In subsection \ref{subsec:larsts}, 
we propose \textit{LARS in tangent space}, 
which is an extension of the original LARS to the GLM. 
The proposed method is 
identical to the original LARS when applied to the normal linear model. 
In subsection \ref{subsec:tlasso}, 
we propose other methods that are related with LASSO. 
Subsection \ref{subsec:others} explains the difference 
between the proposed and the existing methods.

\subsection{Information Geometry}
\label{subsec:ig}

We briefly introduce some tools from information geometry, 
including model manifold, tangent space, and exponential map (Figure \ref{fig:tangent}). 
For details, see 
\cite{A1985, A2016, AN2000, AJLS2017, KV1997, MR1993}. 

In the generalized linear regression, 
we need to select one distribution from 
the exponential family \eqref{eq:glm}.  
The parameter $\bm \theta$ works as a coordinate system in the manifold $\cal M$. 

The tangent space $T_p {\cal M}$ at a point $p \in {\cal M}$ is a linear space consisting of directional derivatives, that is, 
$T_p {\cal M} = \{ v = \sum_{i=1}^d v^i \partial_i |\, v^i \in \mathbb{R}\}$, 
where $\partial_i := \partial/\partial \theta^i$. 
We consider the tangent space $T_{p(\cdot|\,0)} {\cal M}$ at $p(\cdot|\,0)$. 
For simplicity, we call $p(\cdot|\,0)$ and $T_{p(\cdot|\,0)} {\cal M}$, 
the origin and the tangent space $T_0 {\cal M}$ at the origin, respectively. 

Any pair of two vectors in $T_0 {\cal M}$ has its inner product. 
The inner product is determined by the Fisher information matrix $G=G(0)=(g_{ij}(0))$: 
\begin{equation*}
g_{ij}({\bm \theta}) = \mathrm{E}\left[ \partial_i l({\bm \theta}) \partial_j l({\bm \theta}) \right], 
\end{equation*}
where $l({\bm \theta}) = \log p(y|\,{\bm \theta})$ is the log-likelihood. 
Using the Fisher metric $G$, 
the inner product of $v_1 = \sum_{i=1}^d v_1^i \partial_i$ and $v_2 = \sum_{i=1}^d v_2^j \partial_j$ 
is given by 
\begin{equation*}
\langle v_1, v_2 \rangle 
= \sum_{i=1}^d \sum_{j=1}^d v_1^i v_2^j \langle \partial_i, \partial_j \rangle 
= \sum_{i=1}^d \sum_{j=1}^d v_1^i v_2^j g_{ij}. 
\end{equation*}

In the generalized linear regression, 
the Fisher metric $G$ at $T_0 {\cal M}$ is proportional to the correlation matrix $X^\top X$ of the explanatory variables, 
that is, $G = c X^\top X$ for some $c > 0$. 
This is why we use the tangent space at the origin. 
For details, see 
subsection \ref{subsec:metric}.

A point in the tangent space $T_0 {\cal M}$ can be identified with a point in $\cal M$ 
via an exponential map. 
We introduce the e-exponential map 
$\mathrm{Exp}_0: T_0 {\cal M} \rightarrow {\cal M}$, 
which is defined as follows. 
For $v = \sum v^i \partial_i \in T_0 {\cal M}$, 
let $\mathrm{Exp}_0(v) = p(\cdot|\,{\bm v}) \in {\cal M}$ with ${\bm v} = (v^i)$. 
Our problem is estimation for the GLM and 
the parameter is the regression coefficient vector ${\bm \theta} \in \mathbb{R}^d$. 
Therefore, we can avoid technical difficulties of an exponential map. 
The map $\mathrm{Exp}_0$ is a bijection from $T_0 {\cal M}$ to ${\cal M}$. 
For details, see subsection \ref{subsec:e-exp}. 

\begin{figure}
\begin{minipage}{0.5\hsize} 
\begin{center}
  \includegraphics[width=\textwidth]{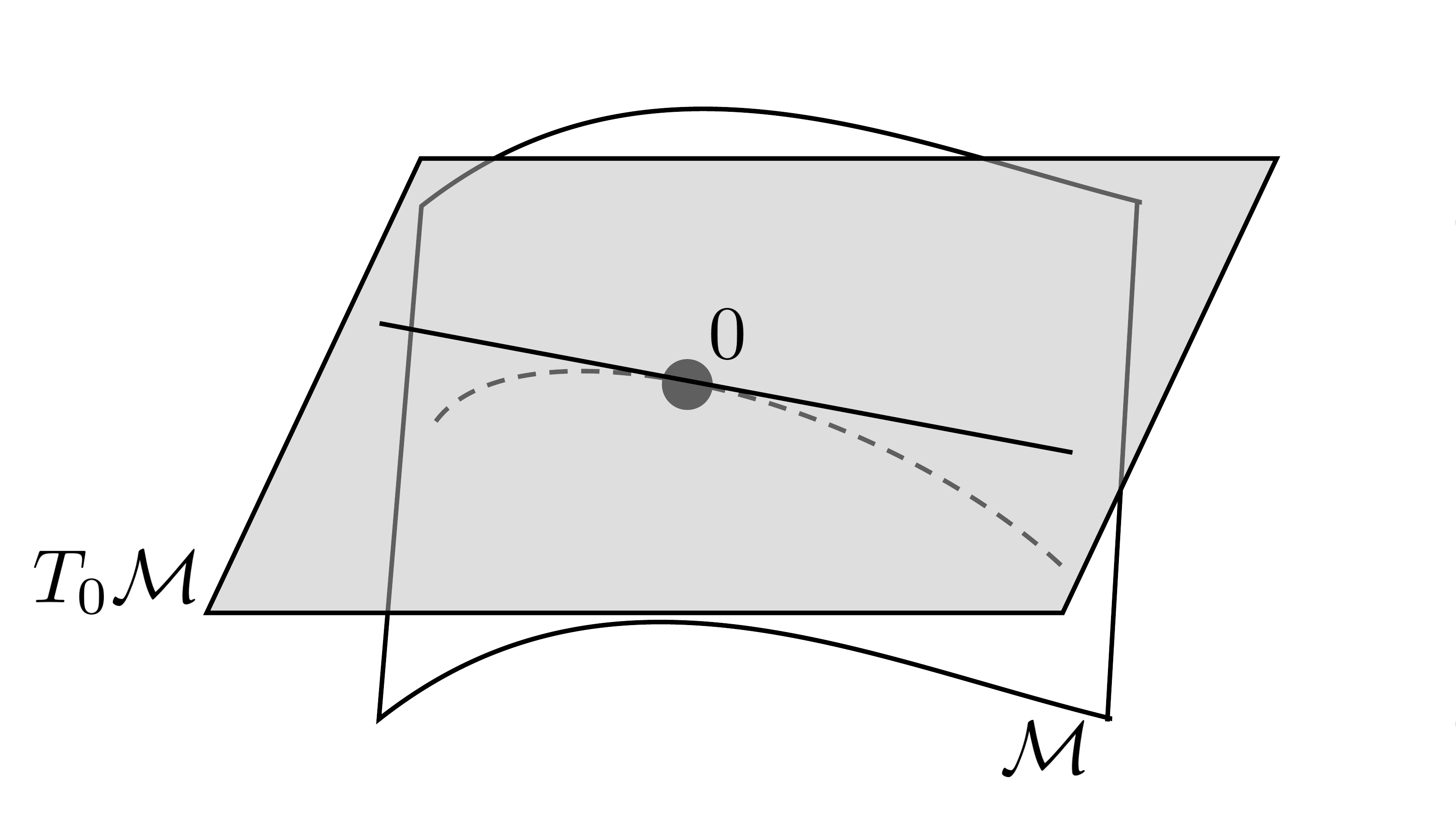}
  \end{center}
  \vspace{-5mm}
  \subcaption{The standard flatness perspective}
  \label{fig:manifold1}
 \end{minipage}
\begin{minipage}{0.5\hsize} 
\begin{center}
  \includegraphics[width=\textwidth]{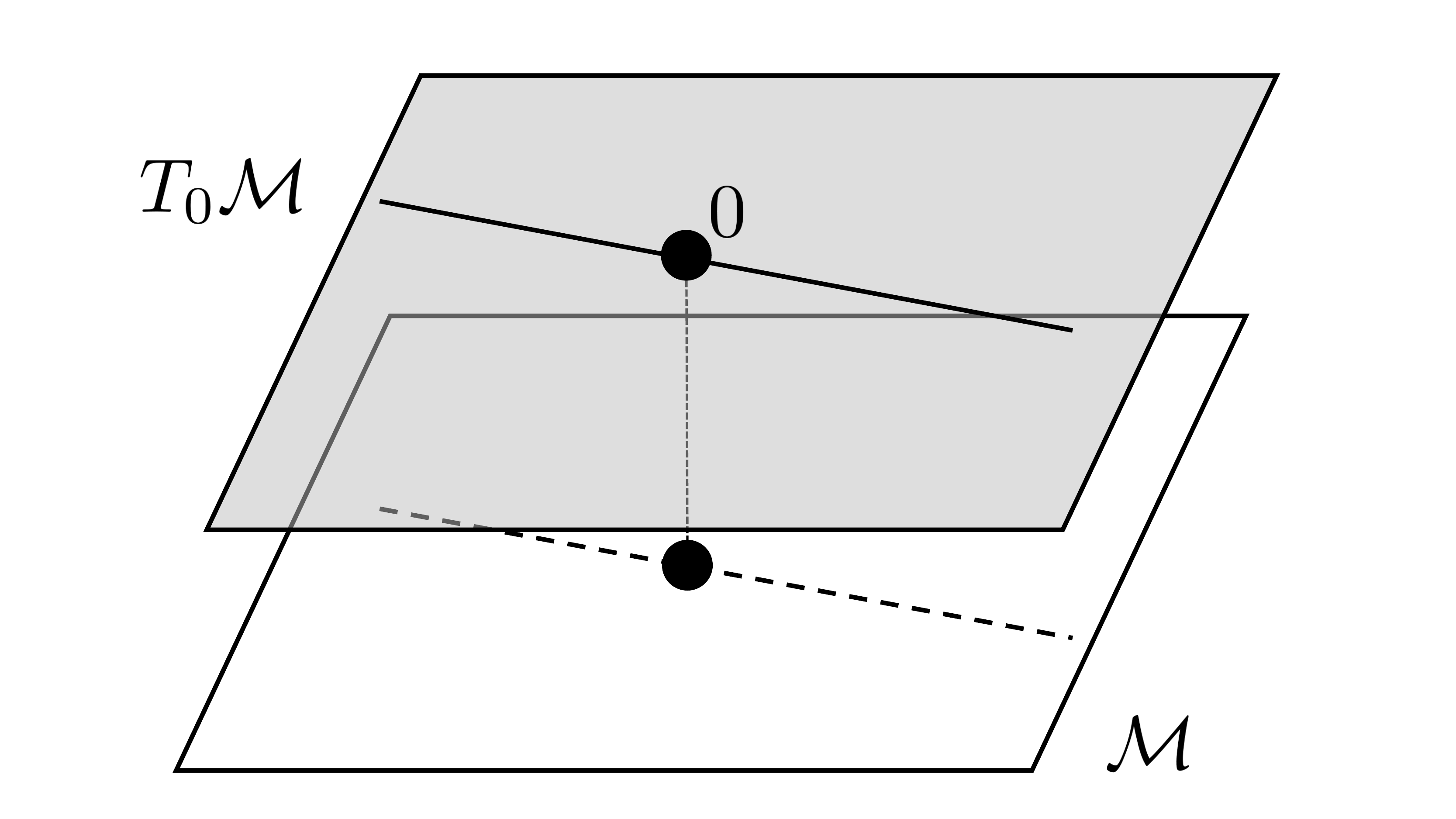}
  \end{center}
  \vspace{-5mm}
  \subcaption{The e-connection perspective}
  \label{fig:manifold2}
 \end{minipage}  
\caption{A statistical manifold $\cal M$ and the tangent space $T_0{\cal M}$ at the origin. 
The white surface is $\cal M$ and the gray plane is $T_0{\cal M}$. 
$\cal M$ is curved from the standard perspective while 
it is flat from the e-connection perspective. 
A point in ${\cal M}$ corresponds to a point in $T_0 {\cal M}$ through the e-exponential map. 
Furthermore, a curve (strictly an e-geodesic) in ${\cal M}$ corresponds to a line in $T_0 {\cal M}$. 
The former is a broken line and the latter is a solid line in the figure 
}
\label{fig:tangent}
\end{figure}

For readers familiar with information geometry, 
we make an additional remark. 
The model manifold $\cal M$ of the GLM is e-flat 
and the regression coefficient ${\bm \theta}$ is an e-affine coordinate system of $\cal M$. 
$\{ \partial_i \}$ is the natural basis of $T_0 {\cal M}$ with respect to 
the coordinate system ${\bm \theta}$. 
Each coordinate axis of $\theta^i$ in $\cal M$ corresponds to $\partial_i$-axis in $T_0 {\cal M}$ via the e-exponential map. 

In the following, 
we also use another representation of $T_0 {\cal M}$. 
This representation is useful for our purpose: 
$T_0 {\cal M} = \{ X {\bm \theta} |\, {\bm \theta} \in \mathbb{R}^d\}$. 
In our notation, $X {\bm \theta}$ also indicates $\sum \theta^i \partial_i$ in the tangent space $T_0 {\cal M}$, 
not only a point $p(\cdot|\, {\bm \theta}) \in {\cal M}$. 
However, we believe it is clear 
because a vector in the tangent space and a point in $\cal M$ 
are identified through the exponential map $\mathrm{Exp}_0$. 

\subsection{LARS in Tangent Space}
\label{subsec:larsts}

The main idea of the proposed method is to run LARS in the tangent space $T_0 {\cal M}$. 
First, we correspond the model manifold to the tangent space $T_0 {\cal M}$ 
using the e-exponential map. 
After this mapping, the original LARS algorithm is used for our computation. 
However, 
we do not use the response ${\bm y}$ directly; 
we introduce 
a virtual response $\hat{\bm y}$. 
The LARS algorithm outputs a sequence of parameter estimates, 
the length of which is the same as the dimension of the parameter $\bm \theta$. 
Finally, the estimates are mapped to the model manifold.

Before running the original LARS algorithm, 
we introduce the virtual response $\hat{\bm y}$. 
The virtual response $\hat{\bm y}$ is defined using the design matrix $X$ 
and the MLE $\hat{\bm \theta}_{\mathrm{MLE}}$ of the full model: 
$\hat{\bm y} = X\hat{\bm \theta}_{\mathrm{MLE}}$. 
Note that LARS uses only correlation coefficients between the response $\bm y$ and the explanatory variables $X$ in the form of ${\bm y}^\top X {\bm \theta}$, 
which is identical with $\hat{\bm \theta}_{\mathrm{MLE}}^\top X^\top X {\bm \theta}$. 
Therefore, introducing the appropriate representation 
$\hat{\bm y} = X\hat{\bm \theta}_{\mathrm{MLE}}$
of the response $\bm y$,
we need only $X^\top X$ as $\hat{\bm y}^\top X {\bm \theta} 
= \hat{\bm \theta}_{\mathrm{MLE}}^\top X^\top X {\bm \theta}$.

In the estimation step of the proposed method, 
we run the original LARS algorithm in the tangent space $T_0 {\cal M}$ 
as if the response is $\hat{\bm y}$. 
LARS outputs a sequence $\{ \hat{\bm \theta}_{(0)}, \hat{\bm \theta}_{(1)}, 
\dots, \hat{\bm \theta}_{(d)} \}$ of the model parameter $\bm \theta$. 
As shown in Figure \ref{fig:lars_f}, 
the LARS estimator $\hat{\bm \theta}$ can be regarded as moving from the origin to the MLE $\hat{\bm \theta}_{\mathrm{MLE}}$ of the full model. 
At the same time, however, the residual 
${\bm r}(\hat{\bm \theta}) := \hat{\bm \theta}_{\mathrm{MLE}} - \hat{\bm \theta}$ 
of the estimator $\hat{\bm \theta}$ is moving 
from the MLE $\hat{\bm \theta}_{\mathrm{MLE}}$ to the origin (Figure \ref{fig:lars_b}). 
The latter is useful for our method 
because it allows us to fix the estimator's tangent space to the origin. 
The residual ${\bm r}(\hat{\bm \theta})$ moves, not the estimator $\hat{\bm \theta}$. 
Note that Algorithm \ref{algo:lars} in subsection \ref{subsec:lars} is actually described from the latter perspective.

\paragraph{LARS in Tangent Space (TLARS)}
LARS in tangent space (TLARS) is given as follows: 

\begin{enumerate}
\item 
Calculate the MLE $\hat{\bm \theta}_{\mathrm{MLE}}$ of the full model.   
\item 
Run the LARS algorithm for the design matrix $X$ 
and the response $\hat{\bm y} = X\hat{\bm \theta}_{\mathrm{MLE}}$. 
\item 
Using the sequence $\{ \hat{\bm \theta}_{(0)}, \hat{\bm \theta}_{(1)}, 
\dots, \hat{\bm \theta}_{(d)} \}$ made by LARS, 
the result is the sequence $\{ p(\cdot|\,\hat{\bm \theta}_{(0)}), 
p(\cdot|\,\hat{\bm \theta}_{(1)}), \dots, p(\cdot|\,\hat{\bm \theta}_{(d)}) \}$. 
\end{enumerate}

As a special case, 
the proposed method coincides with the original LARS 
when we consider the normal linear regression with a known variance. 
Note that TLARS is as computationally efficient as LARS 
although it solves the estimation problem of the GLM. 
Furthermore, we can use existing packages of LARS for the computation of TLARS.

\subsection{LASSO in Tangent Space}
\label{subsec:tlasso}

We propose two estimation methods. 
One is a LASSO-type modification of TLARS 
and the other is an approximation of the $l_1$-regularization for the GLM \eqref{eq:lassoglm}. 

\paragraph{LASSO in Tangent Space 1 (TLASSO1)}
By modifying the LARS algorithm so that it outputs the LASSO estimator 
\cite{EHJT2004}, 
we can use LASSO in the tangent space $T_0 {\cal M}$. 
LASSO in tangent space (TLASSO1) is formally defined as a minimization problem 
\begin{equation}
\label{eq:tlasso1}
\min_{{\bm \theta} \in \mathbb{R}^d} 
\left\{ \| X\hat{\bm \theta}_{\mathrm{MLE}}-X{\bm \theta} \|_2^2 
+ \lambda \| {\bm \theta} \|_1 \right\}, 
\end{equation}
which implies that we use the design matrix $X$ and 
the response $\hat{\bm y} = X\hat{\bm \theta}_{\mathrm{MLE}}$ in the ordinary LASSO. 
This corresponds to the LASSO modification of TLARS.

\paragraph{LASSO in Tangent Space 2 (TLASSO2)}
Another LASSO-type method is a direct approximation of 
\eqref{eq:lassoglm}. 
TLASSO2 is defined as 
\begin{equation}
\label{eq:tlasso2a}
\min_{{\bm \theta} \in \mathbb{R}^d} 
\left\{ \| \alpha X\tilde{\bm \theta}-X{\bm \theta} \|_2^2 + \lambda \| {\bm \theta} \|_1 \right\}, 
\end{equation}
where $\alpha = 1/(h^{-1})'(0)$ and $\tilde{\bm \theta}$ satisfies 
$X^\top X\tilde{\bm \theta} = X^\top {\bm y}$. 
Since the column vectors of the design matrix $X$ are assumed to be linearly independent, 
$\tilde{\bm \theta}$ uniquely exists. 
Problem \eqref{eq:tlasso2a} is LASSO for the normal linear regression 
with the design matrix $X$ and the response $\alpha X\tilde{\bm \theta}$. 
TLASSO2 \eqref{eq:tlasso2a} 
is an approximation of \eqref{eq:lassoglm}. 
In fact, 
using $\tilde{\bm \theta}$ and $\alpha$, 
the log-likelihood is approximated as follows (see subsection \ref{subsec:mleglm}): 
\begin{equation*}
\log p(y|\, {\bm \theta}) 
\approx 
-\frac{1}{2\alpha}({\bm \theta} - \alpha\tilde{\bm \theta})^\top
X^\top X ({\bm \theta} - \alpha\tilde{\bm \theta}) 
+ \frac{\alpha}{2} \tilde{\bm \theta}^\top \tilde{\bm \theta} - \psi(0). 
\end{equation*}
Note that $\alpha \tilde{\bm \theta}$ is an approximation of 
the MLE $\hat{\bm \theta}_{\mathrm{MLE}}$.

\subsection{Remarks on Other Information-Geometrical Methods}
\label{subsec:others}

We briefly compare TLARS with two existing methods 
that are extensions of LARS  based on information geometry. 
One is bisector regression (BR) by 
\cite{HK2010} 
and the other is differential-geometric LARS (DGLARS) by 
\cite{AMW2013}. 
Our concern here is about algorithm itself. 

First, the BR algorithm is very different from TLARS. 
BR takes advantage of the dually flat structure of the GLM and 
attempts to form an equiangular curve using the KL divergence. 
Furthermore, 
the BR estimator moves from the MLE of the full model to the origin 
while, in our method, the residual moves from $\hat{\bm \theta}_{\mathrm{MLE}}$ to the origin. 

DGLARS is also different from TLARS. 
It uses tangent spaces, 
where the equiangular vector is considered. 
However, the DGLARS estimator actually moves from $p(\cdot|\, 0)$ 
to $p(\cdot|\,\hat{\bm \theta}_{\mathrm{MLE}})$ in $\cal M$. 
Accordingly, 
the tangent space at the current estimator moves, 
meaning that we treat the tangent spaces at many points in $\cal M$. 
DGLARS treats the model manifold directly. 
Therefore, it requires many iterations of approximation computation for the algorithm. 
Note that, on the other hand, the update of the TLARS estimator is described fully 
in terms of only the tangent space $T_0 {\cal M}$.

\section{Numerical Examples}
\label{sec:numerical}

We present results of numerical examples and compare our methods with a related method. 
In detail, we compare four methods in the logistic regression setting: 
LARS in Tangent Space (TLARS), 
LASSO in Tangent Space (TLASSO1 and 2), 
and the $l_1$-regularized maximum likelihood estimation for the GLM (L1). 

Our methods do not require an extra implementation 
since the LARS algorithm has already been implemented in the \texttt{lars} package 
of the software R. 
Using R, 
we only needed \texttt{glm()} for calculating the MLE 
and the \texttt{lars} package for the proposed methods. 
For the computation of $l_1$-regularization, 
we used the \texttt{glmnet} package 
\cite{FHT2008}. 

\subsection{Real Data}
\label{subsec:realdata}

We applied the proposed methods and the L1 method to real data. 
We used the South Africa heart disease (SAheart) data 
included in the \texttt{ElemStatLearn} package of R. 
The data contains nine explanatory variables of 462 samples. 
The response is a binary variable. 

We report the results by the four methods. 
Figures \ref{fig:tlars} and \ref{fig:tlasso1} are the paths by TLARS and TLASSO1, respectively. 
In this example, they are the same. 
Figure \ref{fig:tlasso2} is the TLASSO2 path, 
and Figure \ref{fig:glmnet} is the L1 path. 
The paths by TLARS, TLASSO1, and TLASSO2 are made by the \texttt{lars()} function of R, 
and that of L1 by \texttt{glmnet()}. 

As Figure \ref{fig:example} shows, 
the four paths are very similar. 
The proposed methods are based only on the tangent space, 
not on the model manifold itself, 
while L1 directly takes advantage of the likelihood. 
These results imply that the approximation of the model does not require deterioration of the results 
for our methods, especially, for TLARS and TLASSO1. 

\begin{figure}[t]
 \begin{minipage}{0.5\hsize} 
  \begin{center}
  \includegraphics[width=\textwidth]{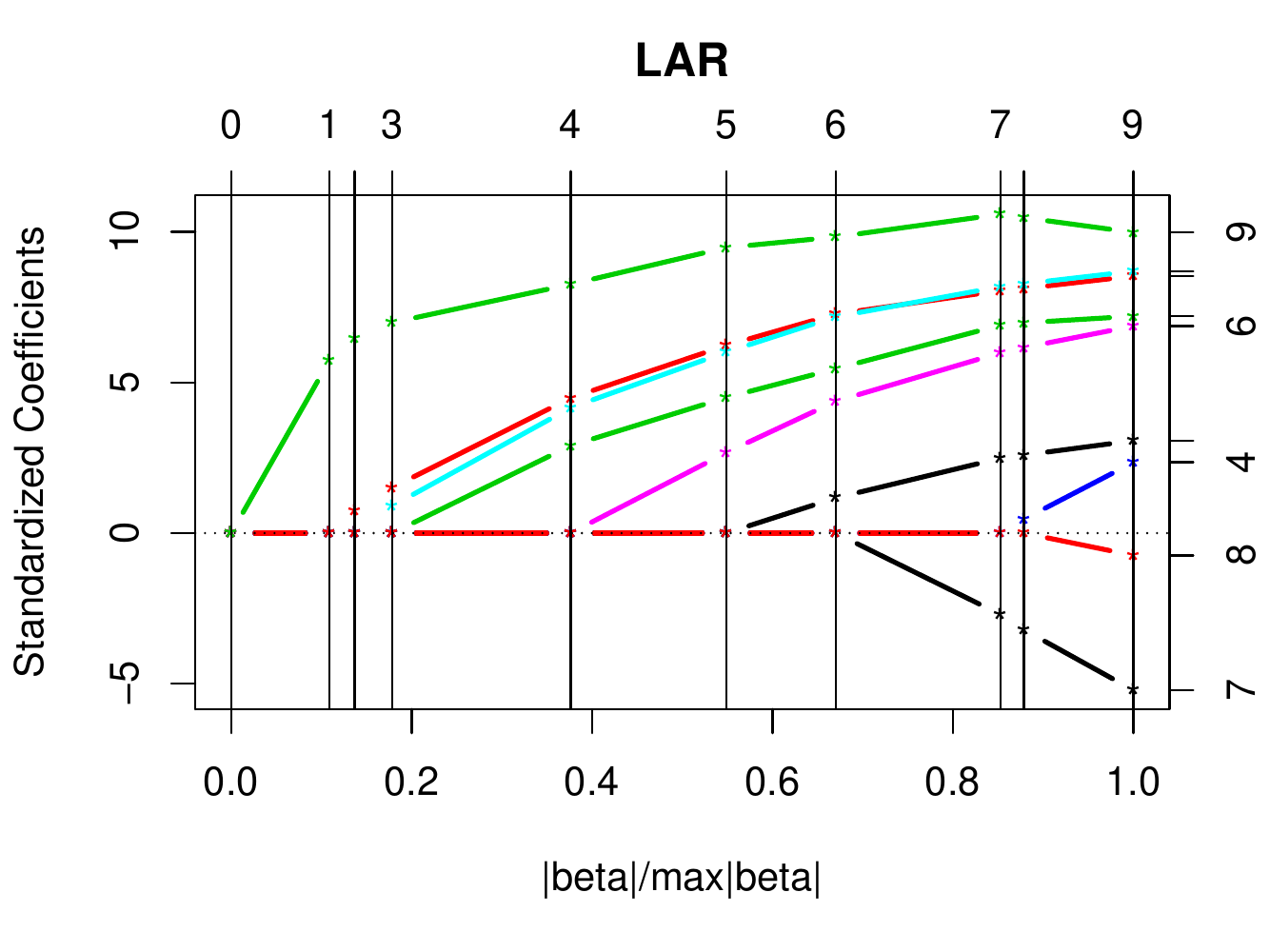}
  \end{center}
  \vspace{-5mm}
  \subcaption{TLARS}
  \label{fig:tlars}
 \end{minipage}
 \begin{minipage}{0.5\hsize}
  \begin{center}
  \includegraphics[width=\textwidth]{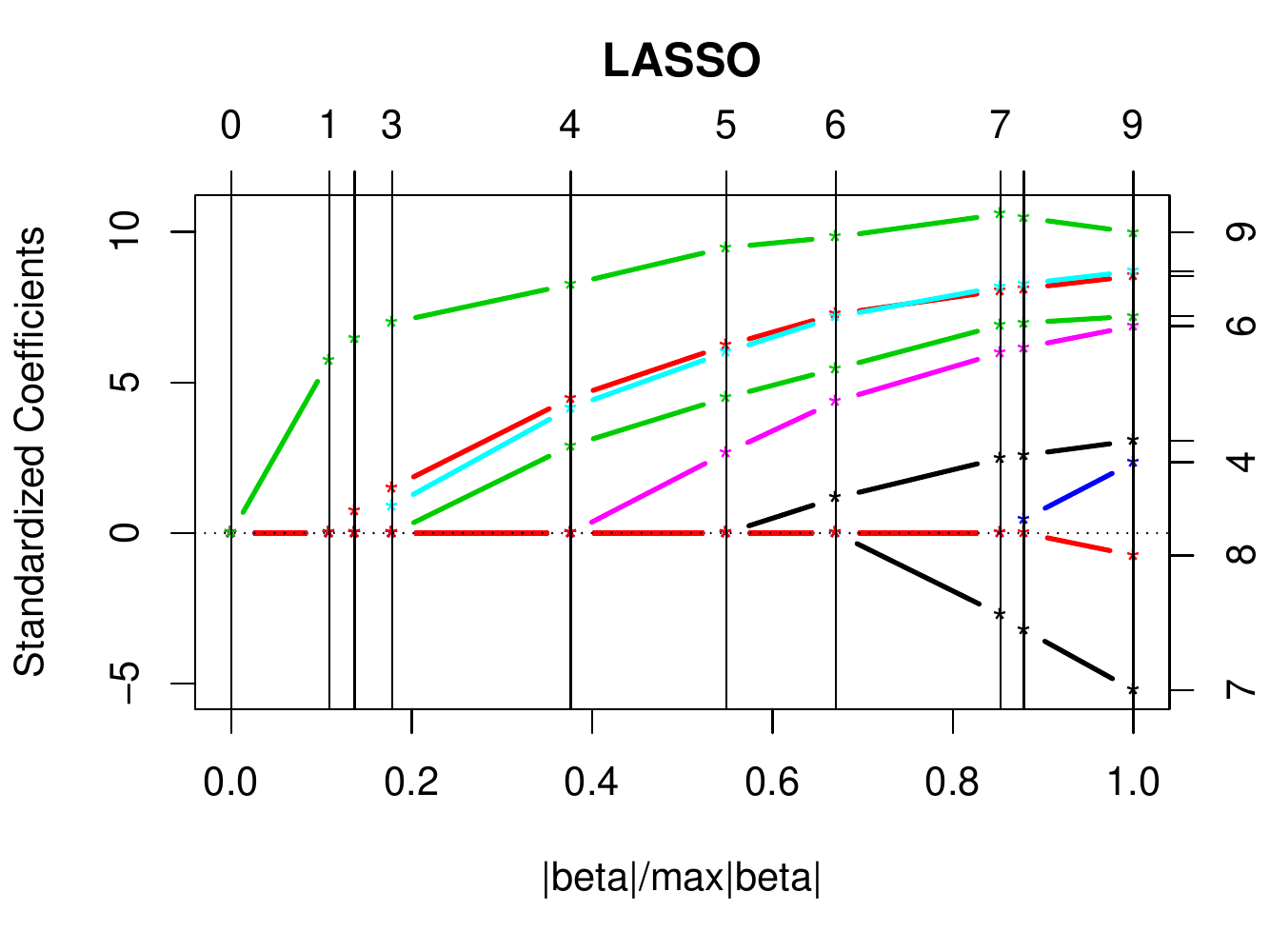}
  \end{center}
  \vspace{-5mm}
  \subcaption{TLASSO1}
  \label{fig:tlasso1}
 \end{minipage} \\
  \vspace{5mm} \\
 \begin{minipage}{0.5\hsize}
  \begin{center}
  \includegraphics[width=\textwidth]{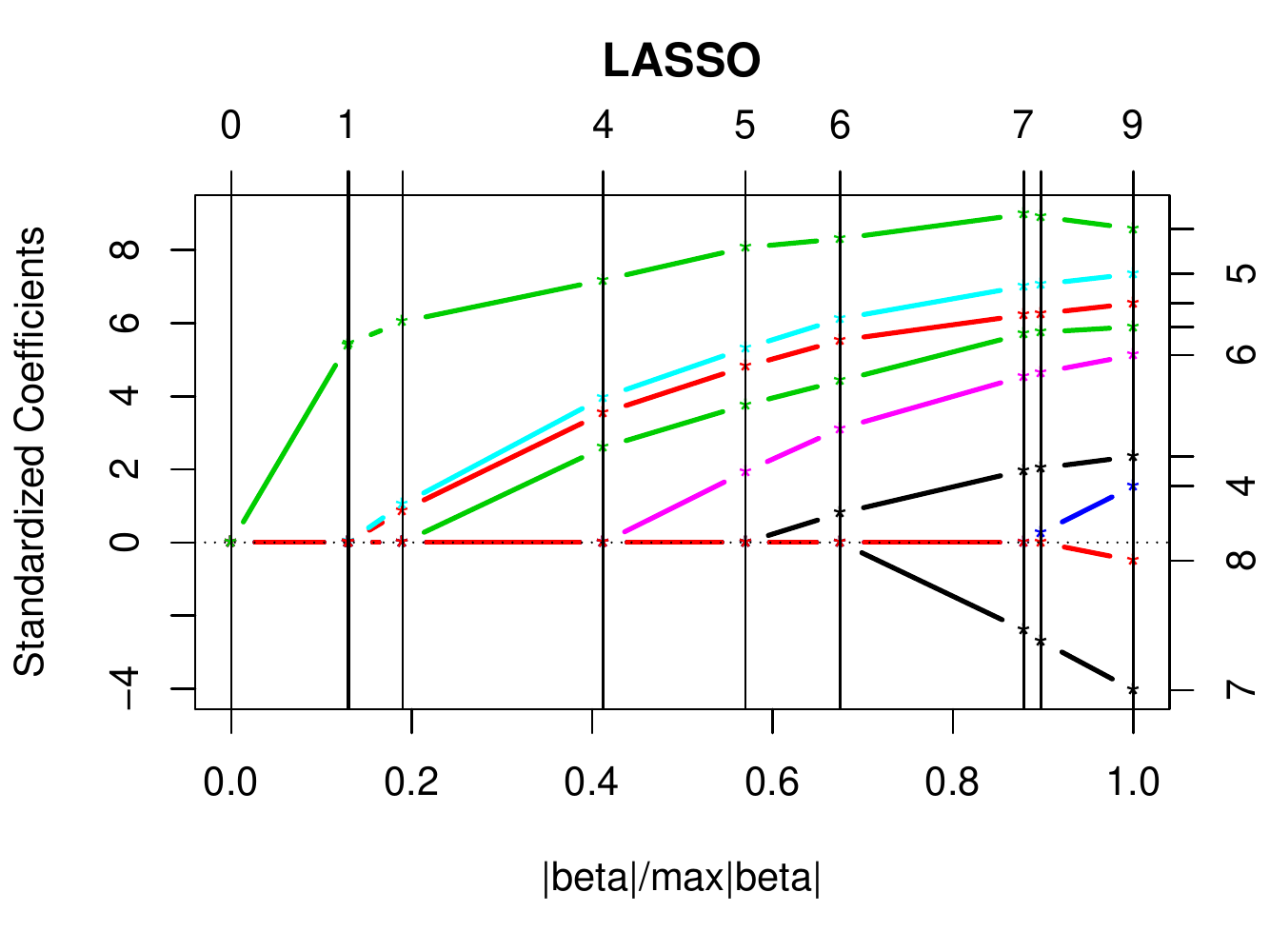}
  \end{center}
  \vspace{-5mm}
  \subcaption{TLASSO2}
  \label{fig:tlasso2}
 \end{minipage}
 \begin{minipage}{0.5\hsize}
  \begin{center}
  \includegraphics[width=\textwidth]{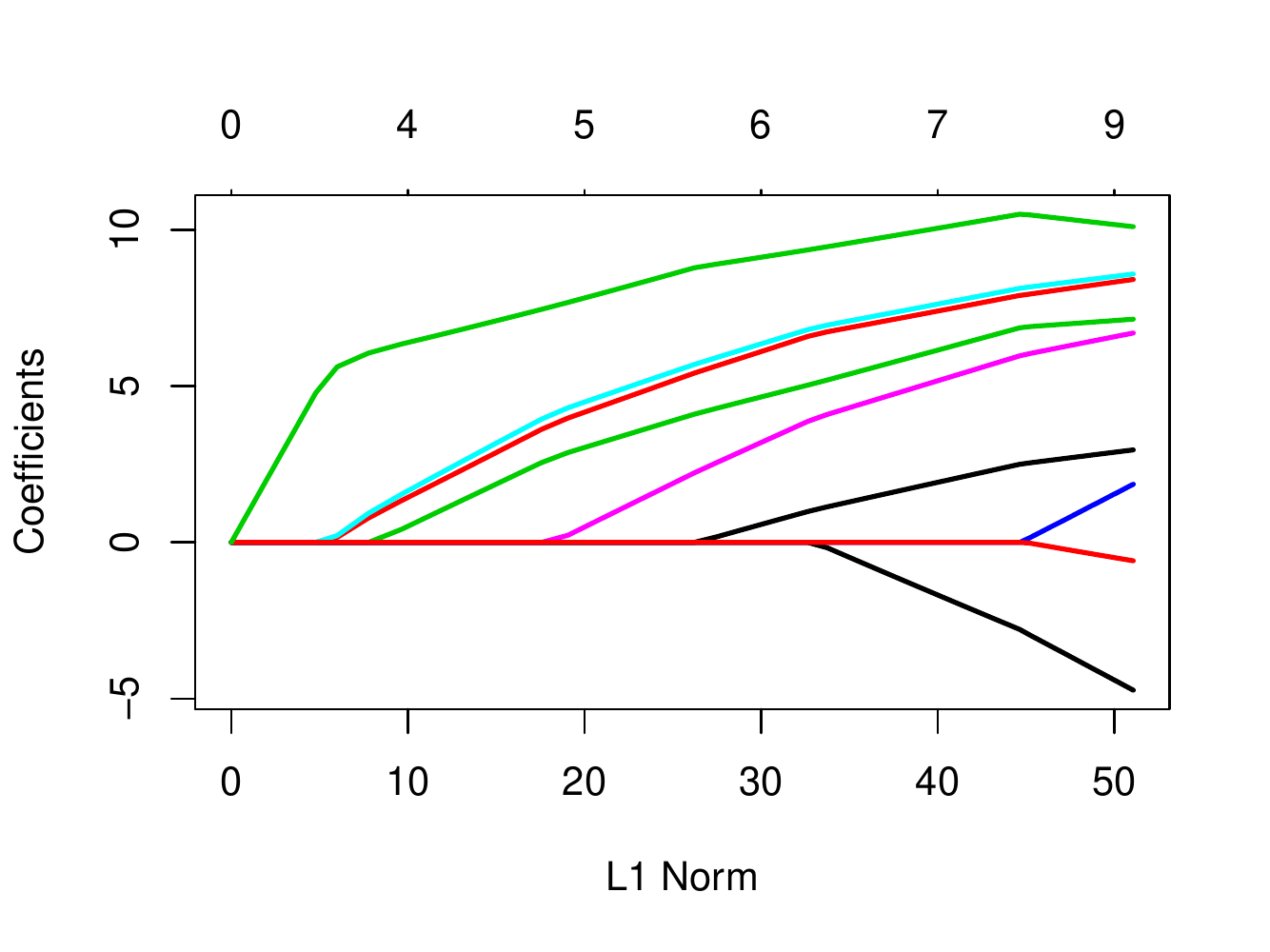}
  \end{center}
  \vspace{-5mm}
  \subcaption{L1}
  \label{fig:glmnet}
 \end{minipage}
 \caption{
 The paths obtained by \ref{fig:tlars}: TLARS, 
 \ref{fig:tlasso1}: TLASSO1, \ref{fig:tlasso2}: TLASSO2, 
 and \ref{fig:glmnet}: L1 are very similar. 
 In this example, the paths by TLARS and TLASSO1 are the same. 
 The paths by TLARS, TLASSO1, and TLASSO2 are made by the \texttt{lars()} function of R, 
 and that of L1 by \texttt{glmnet()}. 
 }
\label{fig:example}
 \end{figure}

\subsection{Numerical Experiments}
\label{subsec:experiment}

We performed numerical experiments of logistic regression. 
The topic is three-fold: generalization, parameter estimation, and model selection. 
The results are presented in Table \ref{tab:numerical}. 
Values in bold are the best and better values. 

The procedure of the experiments is as follows. 
We fixed 
the number of the parameter $d$, 
the true value ${\bm \theta}_0$ of the parameter ${\bm \theta}$, 
and the sample size $n$. 
For each of $m$ trials, 
we made the design matrix $X$ using the \texttt{rnorm()} function in R. 
Furthermore, we made the response $\bm y$ based on $X$ and ${\bm \theta}_0$, 
that is, elements of $\bm y$ have different Bernoulli distributions.   
The four methods were applied to $({\bm y}, X)$. 

For selecting one model and one estimate from a sequence of parameter estimates, 
we used AIC and BIC: 
\begin{align}
\mathrm{AIC} &= -2 \log p(y| \hat{{\bm \theta}}) + 2d', \label{eq:aic} \\
\mathrm{BIC} &= -2 \log p(y| \hat{{\bm \theta}}) + d'\log n, \label{eq:bic}
\end{align}
where $d'$ is the dimension of the parameter of the model under consideration. 
For a sequence $(\hat{\bm \theta}_{(k)})$ made by each of the four methods, 
let $I_{(k)}=\{ i|\, \hat{\theta}_{(k)}^i \not = 0 \}$ and 
$\hat{\bm \theta}_{\mathrm{MLE}}^{(k)}$ the MLE of the model 
${\cal M}_{(k)} = \{ p(\cdot|{\bm \theta})|\, \theta^j = 0 \,\,(j \not \in I_{(k)}) \}$. 
We call 
\eqref{eq:aic} with $\hat{{\bm \theta}}=\hat{\bm \theta}_{\mathrm{MLE}}^{(k)}$ AIC1, 
and \eqref{eq:aic} with $\hat{{\bm \theta}}=\hat{\bm \theta}_{(k)}$ AIC2. 
Similarly, 
\eqref{eq:bic} with $\hat{{\bm \theta}}=\hat{\bm \theta}_{\mathrm{MLE}}^{(k)}$ is BIC1, 
and \eqref{eq:bic} with $\hat{{\bm \theta}}=\hat{\bm \theta}_{(k)}$ is BIC2. 

For evaluating the generalization error of the four methods, 
we newly made $m$ observations 
$\{ ({\bm y}^{l}_1, X^{l}_1), \dots, ({\bm y}^{l}_m, X^{l}_m) \}$ in $l$-th trial ($l=1,2,\dots,m$). 
We computed the difference between $( {\bm y}^{l}_1, \dots, {\bm y}^{l}_m )$ and 
$m$ predictions by each of the methods. 
The ``Generalization'' columns of Table \ref{tab:numerical} 
report the average prediction error over $m$ trials; 
a smaller value is better. 

The ``Model selection'' columns show the proportion of the trials (among $m$ trials) 
where the methods selected the true model. 
The ``Seq'' column indicates the proportion of the trials 
where each sequence of estimates included the true model; 
a larger value is better. 

In the ``Parameter estimation'' columns, 
each value means the average of $\| \hat{\bm \theta} - {\bm \theta}_0 \|_2^2$ 
of the selected estimate $\hat{\bm \theta}$; 
a smaller value is better. 

In Table \ref{tab:numerical}, 
we report the results of three cases. 
We used $m=10,000$ for all cases except case C2, 
for which we set $m=1,000$. 

In case A, 
we set $d=10$ and ${\bm \theta}_0 = (10,10,10,-10,$
$-10,-10,0,0,0,0)^\top$. 
We used $n=100$ for case A1 and $n=1,000$ for A2. 
In generalization, 
three methods (TLARS, TLASSO1, and L1) with AIC2 were much better than 
the other combinations of method and information criterion. 
In model selection, 
the four methods with BIC1 were much better regardless of the sample size. 
In parameter estimation, TLARS and TLASSO1 with AIC1 and BIC2 were better in the small sample setting. 
However, in the larger sample setting, the four methods with AIC2 were better. 
These tendencies were observed in other cases not reported here;  
For example, ${\bm \theta}_0 = (10,10,-10,-10,0,0,0,0,0,0)^\top$. 

Case B is the case of $d=10$ and ${\bm \theta}_0 = (10,10,0,0,0,$ 
$0,0,0,0,0)^\top$ 
with the relation ${\bm x}_3 = {\bm x}_2 + {\bm \epsilon}$, 
where ${\bm x}_2$ and ${\bm x}_3$ are the second and third columns of the design matrix $X$, respectively,  
and $\bm \epsilon$ is distributed according to a multivariate normal distribution. 
We set $n=100$ and $n=1,000$ for cases B1 and B2, respectively. 
In generalization, 
TLARS and TLASSO1 with AIC1, BIC1, and BIC2 were better than the others in case B1. 
Three methods (TLARS, TLASSO1, and L1) with AIC1 and BIC2 were better for the larger sample setting. 
In case B, our interest is mainly in generalization 
because estimation of the true model and the parameter value are not very meaningful. 
However, the four methods with BIC1 were better in model selection. 

In case C, we used $d=50$ and, as ${\bm \theta}_0$, 
the vector of the length 50 with ten $10$s, ten $-10$s, and thirty $0$s. 
In generalization and parameter estimation, 
three methods (TLARS, TLASSO1, and L1) with AIC2 were better than the others 
regardless of the sample size. 
In model selection, 
the four methods with BIC1 were much better than the others. 

In summary, 
the proposed methods  
worked very well. 
Of course, the L1 method sometimes performs better than our methods. 
However, the proposed methods, especially TLARS and TLASSO1, are better than L1 in many situations. 
Furthermore, TLARS and TLASSO1 output the same results in very many trials.

\begin{sidewaystable}
{ 
  \begin{center}
    \caption{ 
    The results of the numerical experiments.
Generalization: 
the average prediction error. 
Model selection: 
the proportion of the trials 
where the methods selected the true model. 
Seq: 
the proportion of the trials in which each sequence of estimates included the true model.  
Parameter estimation: 
the average of the squared error  
of the selected estimate. 
Values in bold are the best and better values. 
}
   \label{tab:numerical}
\small 
    \begin{tabular}{|c|c|rrrr|l|llll|rrrr|} \hline
	\multirow{2}{*}{} & \multirow{2}{*}{Method} & \multicolumn{4}{|c|}{Generalization\, ($\times 10^{-2}$)} & \multicolumn{5}{|c|}{Model selection} & \multicolumn{4}{|c|}{Parameter estimation} \\ \cline{3-15}
	 & & AIC1 & AIC2 & BIC1 & BIC2 & Seq & AIC1 & AIC2 & BIC1 & BIC2 & AIC1 & AIC2 & BIC1 & BIC2 \\ \hline\hline
	\multirow{4}{*}{A1} & {TLARS} & 10.70 & {\bf 9.80} & 12.97 & 10.72 & {0.7246} & 0.3969 & 0.1838 & {\bf 0.4973} & 0.3672 & {\bf 168.3} & 178.7 & 195.5 & {\bf 167.4} \\ \cline{2-15}
	 & {TLASSO1} & 10.70 & {\bf 9.80} & 12.97 & 10.72 & {0.7247} & 0.3968 & 0.1838 & {\bf 0.4974} & 0.3784 & {\bf 168.3} & 178.7 & 195.5 & {\bf 167.4} \\ \cline{2-15}
	 & {TLASSO2} & 15.73 & 12.49 & 18.35 & 15.04 & {0.7086} & 0.4062 & 0.0662 & {\bf 0.4865} & 0.2769 & 249.3 & 171.4 & 310.2 & 232.4 \\ \cline{2-15}
	 & {L1} & 18.81 & {\bf 9.74} & 22.60 & 10.46 & 0.6897 & 0.3996 & 0.0301 & {\bf 0.4824} & 0.1548 & 315.7 & 183.5 & 404.5 & {\bf 169.1} \\ \hline\hline
	\multirow{4}{*}{A2} & TLARS & 4.04 & {\bf 3.60} & 5.14 & 3.96 & {0.9785} & 0.4955 & 0.1252 & {\bf 0.8573} & 0.4988 & 58.7 & {\bf 45.6} & 99.6 & 56.4 \\ \cline{2-15}
	& TLASSO1 & 4.04 & {\bf 3.58} & 5.14 & 3.96 & {0.9785} & 0.4955 & 0.1252 & {\bf 0.8573} & 0.4988 & 58.7 & {\bf 45.6} & 99.6 & 56.4 \\ \cline{2-15}
	& TLASSO2 & 4.73 & 3.69 & 5.91 & 4.40 & {0.9787} & 0.4959 & 0.0561 & {\bf 0.8575} & 0.4022 & 79.6 & {\bf 47.2} & 126.9 & 69.0 \\ \cline{2-15}
	& L1 & 8.20 & {\bf 3.59} & 10.57 & 4.05 & {0.9732} & 0.4968 & 0.0721 & {\bf 0.8570} & 0.3810 & 234.3 & {\bf 45.4} & 352.9 & 58.7 \\ \hline\hline
	\multirow{4}{*}{B1} & {TLARS} & {13.52} & 13.89 & {\bf 13.42} & {\bf 13.16} & 0.5500 & 0.1799 & 0.1455 & {\bf 0.4293} & 0.3369 & 146.7 & 221.6 & 104.9 & 106.5 \\ \cline{2-15}
	 & {TLASSO1} & {\bf 13.33} & 13.67 & {\bf 13.28} & {\bf 13.00} & 0.5643 & 0.1784 & 0.1508 & {\bf 0.4316} & 0.3474 & 144.0 & 214.2 & {\bf 102.8} & {\bf 102.1} \\ \cline{2-15}
	 & {TLASSO2} & 14.05 & 13.94 & 14.81 & 14.23 & 0.5666 & 0.1820 & 0.1152 & {\bf 0.4366} & 0.3257 & 105.0 & 133.8 & 106.2 & {\bf 100.0} \\ \cline{2-15}
	 & {L1} & 15.98 & 14.43 & 19.16 & {13.59} & 0.5560 & 0.1814 & 0.0785 & {\bf 0.4342} & 0.2671 & 131.4 & 334.1 & 155.9 & {\bf 101.4} \\ \hline\hline
	\multirow{4}{*}{B2} & {TLARS} & {\bf 4.95} & 5.20 & 5.16 & {\bf 4.95} & 0.5848 & 0.1926 & 0.1127 & {\bf 0.5402} & 0.4157 & 96.8 & 140.6 & 89.9 & {\bf 84.8} \\ \cline{2-15}
	& TLASSO1 & {\bf 4.95} & 5.20 & 5.16 & {\bf 4.94} & 0.5852 & 0.1918 & 0.1127 & 0.{\bf 5400} & 0.4159 & 96.8 & 140.6 & 89.9 & {\bf 84.8} \\ \cline{2-15}
	& TLASSO2 & 5.00 & 5.31 & 5.29 & 5.05 & 0.5850 & 0.1926 & 0.0978 & {\bf 0.5400} & 0.4104 & 95.0 & 143.7 & 90.8 & {\bf 85.1} \\ \cline{2-15}
	& L1 & 5.90 & 5.10 & 7.62 & {\bf 4.90} & 0.5793 & 0.1925 & 0.0935 & {\bf 0.5367} & 0.3946 & 109.1 & 132.2 & 158.8 & {\bf 82.0} \\ \hline\hline
	\multirow{4}{*}{C1} & {TLARS} & 10.18 & {\bf 9.27} & 14.71 & 10.97 & 0.1479 & 0.0087 & 0.0009 & {\bf 0.0787} & 0.0237 & {\bf 373.3} & {\bf 324.8} & 751.4 & 435.9 \\ \cline{2-15}
	 & {TLASSO1} & 10.18 & {\bf 9.27} & 14.71 & 10.97 & 0.1479 & 0.0087 & 0.0009 & {\bf 0.0787} & 0.0237 & {\bf 373.3} & {\bf 324.8} & 751.4 & 435.9 \\ \cline{2-15}
	 & {TLASSO2} & 14.78 & 11.65 & 19.09 & 15.72 & 0.1399 & 0.0098 & 0.0000 & {\bf 0.0742} & 0.0147 & 811.9 & 537.0 & 1183.8 & 895.2 \\ \cline{2-15}
	 & {L1} & 13.48 & {\bf 9.48} & 19.48 & 12.05 & 0.1137 & 0.0088 & 0.0000 & {\bf 0.0706} & 0.0050 & 690.5 & {\bf 351.6} & 1215.4 & 566.4 \\ \hline\hline
	\multirow{4}{*}{C2} & TLARS & 3.98 & {\bf 3.36} & 6.22 & 4.17 & 0.773 & 0.014 & 0.000 & {\bf 0.486} & 0.077 & 247.4 & {\bf 172.1} & 608.9 & 274.2 \\ \cline{2-15}
	& TLASSO1 & 3.98 & {\bf 3.36} & 6.22 & 4.17 & 0.773 & 0.014 & 0.000 & {\bf 0.486} & 0.077 & 247.4 & {\bf 172.1} & 608.9 & 274.2 \\ \cline{2-15}
	& TLASSO2 & 4.45 & {\bf 3.53} & 6.64 & 4.57 & 0.779 & 0.014 & 0.000 & {\bf 0.486} & 0.068 & 311.6 & 190.4 & 687.7 & 329.0 \\ \cline{2-15}
	& L1 & 4.58 & {\bf 3.40} & 8.08 & 4.34 & 0.736 & 0.015 & 0.000 & {\bf 0.486} & 0.046 & 330.8 & {\bf 176.0} & 982.7 & 297.9 \\ \hline
    \end{tabular}
  \end{center}
}
\end{sidewaystable}

\section{Conclusion}
\label{sec:conclusion}

We proposed the sparse estimation methods as an extension of LARS for the GLM. 
The methods take advantage of the tangent space at the origin, 
which is a natural approximation of the model manifold. 
The proposed methods are computationally efficient because the problem is approximated by the normal linear regression. 
The numerical experiments showed that our idea worked well by comparing the proposed methods with the $l_1$-regularization for the GLM. 
One of our future works is to evaluate our methods theoretically. 
Furthermore, we will apply tools developed for LARS and LASSO to TLARS and TLASSO, 
for example, screening and post-selection inference.

\bibliographystyle{plain}
\bibliography{arxiv2019a_v3}

\appendix

\section{Lemmas and Remarks}
\label{sec:app}

Some lemmas and remarks are presented. 
We use well-known facts on the exponential family, the GLM, and information geometry. 
For details, see 
\cite{A1985, A2016, AN2000, AJLS2017, B1986, KV1997, MN1989, MR1993}. 

As introduced in subsection \ref{subsec:tlasso}, 
$\tilde{\bm \theta}$ satisfies $X^\top X \tilde{\bm \theta} = X^\top {\bm y}$ and $\alpha = 1/\tilde{h}(0)$, 
where $h$ is the link function and $\tilde{h}(0)=(h^{-1})'(0)$. 
Let $\eta_j$ be $j$-th element of the expectation parameter 
${\bm \eta}({\bm \theta}) = \mathrm{E}_{\bm \theta}[X^\top {\bm y}]$. 
Letting ${\bm \mu}({\bm \theta}) = \mathrm{E}_{\bm \theta}[{\bm y}]$, 
it holds ${\bm \eta}({\bm \theta}) = X^\top {\bm \mu}({\bm \theta})$ 
and $\mu^a({\bm \theta}) = h^{-1}(\sum_{i=1}^d x_i^a \theta^i)\,\,(a=1,2,\dots,n)$.

\subsection{Metric at Tangent Space and Correlation Between Explanatory Variables}
\label{subsec:metric}

We show that the Fisher metric $G=(g_{ij})$ at the tangent space $T_0 {\cal M}$ 
is proportional to the correlation matrix $X^\top X$ of the explanatory variables (Lemma \ref{lem:metric}). 
To avoid confusion, in this subsection, 
we use $G(0)=(g_{ij}(0))$ for the metric in $T_0 {\cal M}$ 
and $G({\bm \theta})=(g_{ij}({\bm \theta}))$ for the metric 
in the tangent space at $p(\cdot|\, {\bm \theta})$. 
\begin{lemma}
\label{lem:first}
It holds that 
\begin{align*}
\frac{\partial \psi}{\partial \theta^i}  
= \sum_{a=1}^n x_i^a h^{-1}\Big(\sum_{j=1}^d x_j^a \theta^j \Big). 
\end{align*}
\end{lemma}

\begin{proof}
Since it is known that $\eta_i = \partial \psi/\partial \theta^i$, 
\begin{align*}
\left( \frac{\partial \psi}{\partial \theta^1}, \frac{\partial \psi}{\partial \theta^2}, 
\dots, \frac{\partial \psi}{\partial \theta^d} \right)
&= \left( \eta_1, \eta_2, \dots, \eta_d \right) 
= {\bm \mu}({\bm \theta})^\top X \\
&= \Big( h^{-1}\Big(\sum_{i=1}^d x_i^1 \theta^i \Big), h^{-1}\Big(\sum_{i=1}^d x_i^2 \theta^i \Big), \\
& \hspace{20mm} \dots, 
h^{-1}\Big(\sum_{i=1}^d x_i^n \theta^i \Big) \Big) X. 
\end{align*}
\end{proof}
\begin{lemma}
\label{lem:metric}
$G(0) = c X^\top X$ for some $c > 0$. 
\end{lemma}
\begin{proof}
It is known that the metric $g_{ij}$ is derived from the potential function $\psi$: 
$
g_{ij}({\bm \theta}) = \partial_i \partial_j \psi({\bm \theta}) 
$. 
Therefore, it holds 
\begin{align*}
g_{ij}({\bm \theta}) 
&= \partial_i \partial_j \psi({\bm \theta}) 
= \partial_i \eta_j({\bm \theta}) \\
&= \sum_{a=1}^n x_i^a \partial_j \mu^a({\bm \theta})
= \sum_{a=1}^n x_i^a \partial_j h^{-1}\Big(\sum_{k=1}^d x_k^a \theta^k \Big)\\
&= \sum_{a=1}^n x_i^a x_j^a \tilde{h}\Big(\sum_{k=1}^d x_k^a \theta^k \Big), 
\end{align*}
where $\tilde{h}$ is the derivative of $h^{-1}$. 
Letting ${\bm \theta}=0$ and $c=\tilde{h}(0)$, 
we have $G(0) = c X^\top X$. 
Since both $G(0)$ and $X^\top X$ are known to be positive definite, 
$c$ is a positive constant. 
\end{proof}
Note that $c$ is common to all $i,j$ and $a$ in the proof. 
Hence, the tangent space $T_0 {\cal M}$ at the origin ${\bm \theta} = 0$ is selected 
as the space in which LARS runs.

\subsection{Approximations of the Likelihood and MLE}
\label{subsec:mleglm}

We approximate the log-likelihood and the MLE of the GLM \eqref{eq:glm}. 
Lemma \ref{lem:log-lik} implies that 
$\alpha \tilde{{\bm \theta}}$ is an approximation of the MLE $\hat{\theta}$

\begin{lemma}
\label{lem:log-lik}
The log-likelihood is expanded as 
\begin{equation*}
\log p(y|\, {\bm \theta}) 
= -\frac{1}{2\alpha}({\bm \theta} - \alpha\tilde{\bm \theta})^\top
X^\top X ({\bm \theta} - \alpha\tilde{\bm \theta}) 
+ \frac{\alpha}{2} \tilde{\bm \theta}^\top \tilde{\bm \theta} - \psi(0)  
+ O(\| {\bm \theta} \|^3). 
\end{equation*}
\end{lemma}
\begin{proof}
Using $\alpha = 1/\tilde{h}(0)$ and Lemmas \ref{lem:first} and \ref{lem:metric}, 
the potential function $\psi$ is expanded as follows: 
\begin{align*}
\psi({\bm \theta}) 
&= 
\psi(0) + \left( \frac{\partial \psi}{\partial \theta^1}(0), \frac{\partial \psi}{\partial \theta^2}(0), 
\dots, \frac{\partial \psi}{\partial \theta^d}(0) \right)  {\bm \theta} 
+ \frac{1}{2} {\bm \theta}^\top G(0) {\bm \theta}  + O(\| {\bm \theta} \|^3) \\
&= 
\psi(0) + h^{-1}(0){\bm 1}^\top X{\bm \theta} 
+ \frac{1}{2\alpha} {\bm \theta}^\top X^\top X {\bm \theta} 
+ O(\| {\bm \theta} \|^3) \\
&= \psi(0) + \frac{1}{2\alpha} {\bm \theta}^\top X^\top X {\bm \theta} 
+ O(\| {\bm \theta} \|^3). 
\end{align*}
At the last equal sign, we used ${\bm 1}^\top X = 0$ since each column vector of $X$ is assumed to be normalized. 
Therefore, 
\begin{align*}
\log p(y|\, {\bm \theta}) 
&= {\bm y}^\top X {\bm \theta} - \psi({\bm \theta}) \\
&= \tilde{\bm \theta}^\top X^\top X {\bm \theta} - \psi({\bm \theta}) \\ 
&= \tilde{\bm \theta}^\top X^\top X {\bm \theta} - 
\left\{
\psi(0) 
+ \frac{1}{2\alpha} {\bm \theta}^\top X^\top X {\bm \theta}
+ O(\| {\bm \theta} \|^3)
\right\} \\
&= -\frac{1}{2\alpha}({\bm \theta} - \alpha\tilde{\bm \theta})^\top
X^\top X ({\bm \theta} - \alpha\tilde{\bm \theta}) 
+ \frac{\alpha}{2} \tilde{\bm \theta}^\top \tilde{\bm \theta} - \psi(0) 
+ O(\| {\bm \theta} \|^3). 
\end{align*}
\end{proof}

\subsection{e-Exponential Map}
\label{subsec:e-exp}

In Riemannian geometry, 
a point in a tangent space is mapped to a manifold via an exponential map, 
which is defined using a geodesic. 
A geodesic in a manifold corresponds to a straight line in Euclidean space. 
When we consider an exponential map, 
we need to introduce not only a metric but also a connection, 
which determines flatness and straightness in a manifold.  
In section \ref{sec:larsts}, 
we implicitly introduced the e-connection. 
From the viewpoint of the e-connection, 
each curve of $\theta^i$-axis is an e-geodesic in $\cal M$.  

For a manifold $\cal M$ and a point $p \in {\cal M}$, 
an exponential map $f$ at $p$ is formally defined as follows. 
First, we consider the geodesic $\gamma_v(t)$ for $v \in T_p{\cal M}$ 
which satisfies $\gamma_v(0) = p$ and $\mathrm{d} \gamma_v(t)/\mathrm{d}t|_{t=0} = v$. 
Here, the parameter $t$ moves in an interval including $0$. 
Note that, given a connection, the geodesic $\gamma_v$ locally exists and is uniquely determined.  
The exponential map $f$ is $f: T_p{\cal M} \rightarrow {\cal M}$ and $f(v) = \gamma_v(1)$ for $v \in D \subset T_p{\cal M}$, 
where $D = \{ v \in T_p{\cal M}|\, \gamma_v(1)\,\, \mathrm{exists} \}$.

In general, 
an exponential map is not necessarily easy to treat. 
For example, the domain $D$ of such a map is known as a star-shaped domain 
and does not coincide with a whole tangent space. 
However, our exponential map $\mathrm{Exp_0}: T_0{\cal M} \rightarrow M$ has a useful property.  
The domain of $\mathrm{Exp_0}$ is a whole $T_0{\cal M}$ and the range is a whole $\cal M$. 

\begin{lemma}
\label{lem:exp}
The map $\mathrm{Exp}_0: T_0{\cal M} \rightarrow M$ defined in subsection \ref{subsec:ig} is 
the e-exponential map for a manifold of the GLM. 
Furthermore, $\mathrm{Exp}_0$ is a bijection 
from the tangent space $T_0 {\cal M}$ to the manifold $\cal M$. 
\end{lemma}
\begin{proof}
For $v = \sum_{i=1}^d v^i \partial_i \in T_0 {\cal M}$, 
the value of the map is $\mathrm{Exp_0}(v) = p(\cdot|\,{\bm v})$, 
where ${\bm v} = (v^i)$. 
It is known that the e-geodesic $\gamma(t)$ satisfying $\gamma(0)=p(\cdot|\,0)$ 
and $\mathrm{d} \gamma(t)/\mathrm{d}t|_{t=0}=v \in T_0 {\cal M}$ is represented as 
$\gamma(t)=p(\cdot|\,t{\bm v})$. 
Therefore, $\mathrm{Exp_0}(v) = p(\cdot|\,{\bm v}) = \gamma(1)$, 
which means that $\mathrm{Exp}_0$ is the e-exponential map. 

Since ${\cal M}= \{ p(\cdot|\,{\bm \theta})|\, {\bm \theta} \in \mathbb{R}^d \}$, 
the e-exponential map is defined on a whole $T_0{\cal M}$. 
For ${\bm \theta} \in \mathbb{R}^d$, 
$w=\sum_{i=1}^d \theta^i \partial_i$ is in $T_0 {\cal M}$ and $\mathrm{Exp_0}(w) = p(\cdot|\,{\bm \theta})$, 
which implies that the e-exponential map is a surjection. 
Furthermore, if $v, w \in T_0 {\cal M}$ are different, $\mathrm{Exp_0}(v) \not = \mathrm{Exp_0}(w)$ 
because the column vectors of $X$ are assumed to be linearly independent. 
\end{proof}

\end{document}